\documentclass{article}

\usepackage{PRIMEarxiv}

\usepackage[utf8]{inputenc} 
\usepackage[T1]{fontenc}    
\usepackage{hyperref}       
\usepackage{url}            
\usepackage{booktabs}       
\usepackage{nicefrac}       
\usepackage{microtype}      
\usepackage{lipsum}
\usepackage{fancyhdr}       

\usepackage{cite}
\usepackage{amsmath,amssymb,amsfonts}
\usepackage{algorithmic}
\usepackage{graphicx}
\usepackage{textcomp}
\usepackage{amsthm}
\theoremstyle{definition}
\newtheorem{definition}{Definition}
\newtheorem{theorem}{Theorem}
\theoremstyle{remark}
\newtheorem*{remark}{Remark}
\usepackage{listings}
\usepackage{algorithm}
\usepackage{threeparttable}

\newtheorem{proposition}{Proposition}

\newtheorem{corollary}{Corollary}
\theoremstyle{remark}

\title{Responsible AI: The Good, The Bad, The AI
\thanks{\textit{{This work is submitted for review to Journal of Strategic Information Systems.}}} 
}

\author{
  A. A. Jafari, C. Ozcinar\\
  University of Tartu \\
  Tartu, Estonia \\
  \texttt{\{akbar.anbar.jafari\}@ut.ee} \\
   \And
  G. Anbarjafari \\
  3S Holding OÜ \\
  Tartu, Estonia\\
  \texttt{shb@3sholding.com} \\
}

\begin{document}
\maketitle

\begin{abstract}
The rapid proliferation of artificial intelligence across organizational contexts has generated profound strategic opportunities while introducing significant ethical and operational risks. Despite growing scholarly attention to responsible AI, extant literature remains fragmented and is often adopting either an optimistic stance emphasizing value creation or an excessively cautious perspective fixated on potential harms. This paper addresses this gap by presenting a comprehensive examination of AI's dual nature through the lens of strategic information systems. Drawing upon a systematic synthesis of the responsible AI literature and grounded in paradox theory, we develop the Paradox-based Responsible AI Governance (PRAIG) framework that articulates: (1) the strategic benefits of AI adoption, (2) the inherent risks and unintended consequences, and (3) governance mechanisms that enable organizations to navigate these tensions. Our framework advances theoretical understanding by conceptualizing responsible AI governance as the dynamic management of paradoxical tensions between value creation and risk mitigation. We provide formal propositions demonstrating that trade-off approaches amplify rather than resolve these tensions, and we develop a taxonomy of paradox management strategies with specified contingency conditions. For practitioners, we offer actionable guidance for developing governance structures that neither stifle innovation nor expose organizations to unacceptable risks. The paper concludes with a research agenda for advancing responsible AI governance scholarship.
\end{abstract}


\section{Introduction}
\label{sec:intro}

Artificial intelligence (AI) stands at a critical inflection point in its integration into organizational and societal fabrics. AI has shown a great ability to change how businesses work. It can improve decision-making and create new value. Generative AI could add trillions each year to the global economy \cite{mckinsey2023genai}. Organizations report significant improvements in operational efficiency, customer engagement, and competitive positioning \cite{ransbotham2020winning, enholm2022ai}.

Yet, the same technological capabilities that enable these benefits have simultaneously given rise to profound concerns. Amazon's AI recruitment system systematically discriminated against female candidates \cite{dastin2022amazon}. Algorithmic decision-making in criminal justice perpetuates racial biases \cite{angwin2016machine}. Facial recognition systems raise fundamental questions about privacy and civil liberties \cite{buolamwini2018gender,domnich2021responsible,sham2023ethical}. These incidents are not isolated anomalies but manifestations of systemic tensions inherent in AI deployment.

The fundamental challenge confronting organizations is not whether to adopt AI, but how to do so responsibly while maintaining competitive viability. AI's benefits and risks are deeply intertwined \cite{mikalef2022thinking,mirza2025quantifying,jafari2025mathematical}: the characteristics that make AI strategically valuable---pattern recognition in vast datasets, automated complex decisions, superhuman speed and scale---are precisely what renders it potentially harmful when misaligned with human values \cite{floridi2018ai4people}.

Regulatory landscapes have evolved rapidly. The European Union's AI Act establishes the world's first comprehensive legal framework for AI governance \cite{euaiact2024}. Organizations thus face a paradoxical situation: they must embrace AI to remain competitive while simultaneously managing its risks to maintain legitimacy and comply with regulations \cite{papagiannidis2025responsible}.

\subsection{Research Gap and Contributions}

Despite extensive scholarship on responsible AI, significant fragmentation persists. Technical research develops sophisticated fairness metrics and explainability techniques but often abstracts from organizational realities \cite{mehrabi2021survey}. Governance scholarship proposes structural mechanisms but frequently undertheorizes the tensions inherent in implementation \cite{mittelstadt2019principles}. The resulting ``principles-to-practices gap'' frustrates organizational efforts \cite{schiff2021principles}.

We argue this gap stems from a fundamental conceptual error: treating responsible AI governance as a trade-off optimization problem rather than a paradox management challenge. This reconceptualization motivates our research questions: (1) How should the relationship between AI value creation and responsible deployment be theoretically conceptualized? (2) What strategies effectively manage the tensions inherent in responsible AI governance? (3) How can organizations implement governance mechanisms that address paradoxical tensions?

This paper makes three primary contributions. \textit{First}, we reconceptualize responsible AI governance as paradox management, demonstrating formally that the value-responsibility relationship meets paradox criteria and that trade-off approaches amplify rather than resolve tensions. \textit{Second}, we develop a taxonomy of paradox management strategies---acceptance, temporal separation, spatial separation, and integration---with formally specified contingency conditions. \textit{Third}, we present the PRAIG (Paradox-based Responsible AI Governance) framework, synthesizing antecedents, practices, outcomes, and feedback dynamics into an integrated model.

\section{Background and Literature Review}
\label{sec:literature}

We organize this review around three pillars: ``The Good'' (AI's strategic value), ``The Bad'' (risks and unintended consequences), and ``The AI'' (governance frameworks).

\subsection{The Good: Strategic Value of AI}
\label{subsec:lit_good}

Contemporary conceptualizations of AI capability extend beyond technology to encompass organizational capacity to deploy, manage, and derive value from AI systems \cite{mikalef2021artificial}. The resource-based view suggests that AI capabilities meeting criteria of value, rarity, inimitability, and non-substitutability can generate sustainable competitive advantage \cite{bharadwaj2000resource}.

AI creates business value through multiple mechanisms \cite{enholm2022ai}. \textit{Operational efficiency} emerges from automation of routine cognitive tasks, with documented productivity gains of 15-40\% \cite{davenport2018artificial}. \textit{Enhanced decision-making} leverages AI's capacity to process vast datasets and identify patterns, with demonstrated accuracy improvements in domains from credit risk to medical diagnosis \cite{agrawal2018prediction}. \textit{Customer experience enhancement} derives from AI's personalization capabilities \cite{huang2019artificial}. \textit{Business model innovation} enables entirely new value propositions \cite{sjoedin2021ai}. Beyond direct value, AI adoption enhances organizational agility \cite{sambamurthy2003shaping}, innovation capabilities \cite{verganti2020innovation}, and knowledge management processes \cite{canhoto2021artificial}. Table~\ref{tab:ai_benefits} synthesizes the strategic benefits documented in the literature.

\begin{table}[ht]
\centering
\caption{Strategic Benefits of AI: A Synthesis}
\label{tab:ai_benefits}
\small
\begin{tabular}{p{2.5cm}p{4.5cm}p{5cm}}
\hline
\textbf{Benefit Category} & \textbf{Mechanisms} & \textbf{Illustrative Evidence} \\
\hline
Operational Efficiency & Task automation; Process optimization; Predictive maintenance & 20-35\% cost reduction in routine operations \cite{davenport2018artificial}; 15-25\% improvement in equipment effectiveness \\
\hline
Decision Quality & Pattern recognition; Predictive analytics; Real-time insights & Diagnostic accuracy improvements of 10-20\% in medical imaging \cite{topol2019high}; Enhanced credit risk prediction \cite{cao2022ai} \\
\hline
Customer Experience & Personalization; Conversational AI; Anticipatory service & Increased customer satisfaction and engagement; Improved conversion rates \cite{huang2019artificial} \\
\hline
Innovation & Generative design; Accelerated experimentation; Capability expansion & New product development acceleration \cite{verganti2020innovation}; Novel business models \cite{sjoedin2021ai} \\
\hline
Organizational Agility & Environmental sensing; Rapid response; Adaptive operations & Improved market responsiveness \cite{mikalef2021artificial}; Enhanced strategic flexibility \cite{sambamurthy2003shaping} \\
\hline
\end{tabular}
\end{table}

\subsection{The Bad: Risks and Unintended Consequences}
\label{subsec:lit_bad}

Mikalef et al. \cite{mikalef2022thinking} articulated a comprehensive framework for understanding AI's ``dark side.'' We extend this framework across several dimensions.

\textit{Algorithmic bias} represents one of the most consequential risks. Obermeyer et al. \cite{obermeyer2019dissecting} documented racial bias in healthcare algorithms affecting millions. Buolamwini and Gebru \cite{buolamwini2018gender} demonstrated facial recognition error rates up to 34 times higher for darker-skinned women. Critically, bias often operates invisibly, embedded in technical systems whose workings remain opaque \cite{burrell2016machine}.

\textit{Transparency gaps} undermine multiple objectives: users cannot verify decisions, affected individuals cannot contest outcomes, auditors cannot assess compliance \cite{arrieta2020explainable}. The emerging field of explainable AI has developed techniques for interpretability, but a fundamental trade-off exists between model complexity and interpretability \cite{rudin2019stop}.

\textit{Accountability challenges} arise from AI's distributed development and deployment. When AI causes harm, responsibility diffuses across data providers, developers, deployers, and users \cite{dignum2019responsible}. This ``problem of many hands'' can result in accountability gaps \cite{nissenbaum1996accountability}.

Additional risks include \textit{robustness and safety concerns} (adversarial attacks \cite{goodfellow2014explaining}, distributional shift \cite{quinonero2008dataset}, alignment failures \cite{gabriel2020artificial}), \textit{data governance failures} (privacy violations, quality issues, intellectual property concerns \cite{solove2013privacy}), and \textit{societal impacts} (labor displacement, power concentration, environmental costs \cite{crawford2021atlas, schwartz2020green}).

\subsection{The AI: Governance Frameworks}
\label{subsec:lit_governance}

The responsible AI governance landscape has evolved rapidly. Over 160 sets of ethical principles have been proposed by governments, industry bodies, and academic institutions \cite{jobin2019global}. Floridi et al. \cite{floridi2018ai4people} identified convergence around five principles: beneficence, non-maleficence, autonomy, justice, and explicability.

However, a persistent ``principles-to-practices gap'' undermines implementation \cite{mittelstadt2019principles}. Principles remain abstract while organizations require concrete operational guidance. Papagiannidis et al. \cite{papagiannidis2025responsible} synthesized governance practices into structural (ethics boards, AI officers), procedural (impact assessments, audits), and relational (stakeholder dialogue, training) mechanisms.

We argue this gap reflects not merely implementation challenges but a fundamental conceptual error: treating the value-responsibility relationship as a trade-off amenable to optimization rather than a paradox requiring ongoing management.

\section{Theoretical Framework: Paradox Theory and AI Governance}
\label{sec:theory}

\subsection{Paradox as Theoretical Lens}
\label{subsec:paradox_intro}

Paradox theory has emerged as a powerful lens for understanding persistent organizational tensions that resist resolution \cite{smith2011toward, schad2016paradox}. Unlike dilemmas (choosing between alternatives), trade-offs (optimization along a frontier), or dialectics (eventual synthesis), paradoxes involve ``contradictory yet interrelated elements that exist simultaneously and persist over time'' \cite{smith2011toward}.

The responsible AI domain exhibits paradox's defining characteristics. The tension between value creation and responsible deployment is \textit{contradictory}: aggressive AI deployment may conflict with careful governance. Yet the elements are \textit{interrelated}: irresponsible deployment ultimately destroys value through penalties, reputational damage, and failures, while excessive caution forfeits competitive benefits. Furthermore, tension \textit{persists}: resolving it at one point does not eliminate it, as technological evolution, regulatory change, and competitive dynamics continuously regenerate the contradiction.

\subsection{Formal Constructs}
\label{subsec:formal_constructs}

To provide analytical rigor, we formalize core constructs.

\begin{definition}[AI Deployment Configuration]
\label{def:config}
An \textbf{AI deployment configuration} is a tuple $\mathcal{C} = (T, G, E)$ where $T = \{t_1, \ldots, t_n\}$ is a set of AI technologies deployed, $G = \{g_1, \ldots, g_m\}$ is a set of governance mechanisms implemented, and $E: T \times G \rightarrow \mathbb{R}$ is an effectiveness function mapping technology-governance pairs to outcomes.
\end{definition}

\begin{definition}[Value and Risk Functions]
\label{def:value_risk}
The \textbf{value function} $V: \mathcal{C} \rightarrow \mathbb{R}$ quantifies strategic benefits:
\begin{equation}
V(\mathcal{C}) = \sum_{i=1}^{n} \alpha_i \cdot v(t_i) - \sum_{j=1}^{m} \beta_j \cdot c(g_j) + \sum_{i,j} \gamma_{ij} \cdot E(t_i, g_j)
\label{eq:value}
\end{equation}
The \textbf{risk function} $R: \mathcal{C} \rightarrow \mathbb{R}^+_0$ quantifies potential for harm:
\begin{equation}
R(\mathcal{C}) = \sum_{i=1}^{n} \rho_i \cdot r(t_i) \cdot \prod_{j=1}^{m} \left(1 - \mu_{ij} \cdot g_j\right)
\label{eq:risk}
\end{equation}
where $v(t_i)$ is technology value, $c(g_j)$ is governance cost, $r(t_i)$ is inherent risk, and $\mu_{ij}$ is mitigation effectiveness.
\end{definition}

\begin{definition}[Paradoxical Tension]
\label{def:tension}
A \textbf{paradoxical tension} exists when: (1) \textit{Contradiction}: $\exists \mathcal{C}_1, \mathcal{C}_2$ such that $V(\mathcal{C}_1) > V(\mathcal{C}_2)$ and $R(\mathcal{C}_1) > R(\mathcal{C}_2)$; (2) \textit{Interdependence}: $\frac{\partial V}{\partial R} \neq 0$ and $\frac{\partial R}{\partial V} \neq 0$; (3) \textit{Persistence}: $\nexists \mathcal{C}^*$ such that $V(\mathcal{C}^*) = \max V$ and $R(\mathcal{C}^*) = \min R$ simultaneously.
\end{definition}

\begin{proposition}[Existence of Paradox]
\label{prop:existence}
For any non-trivial AI deployment context (where technology has positive value potential and non-zero inherent risk), a paradoxical tension exists.
\end{proposition}

\begin{proof}
\textit{Contradiction}: Consider $\mathcal{C}_1$ with extensive deployment and minimal governance versus $\mathcal{C}_2$ with limited deployment and extensive governance. From Equation~\eqref{eq:value}, $V(\mathcal{C}_1) > V(\mathcal{C}_2)$; from Equation~\eqref{eq:risk}, $R(\mathcal{C}_1) > R(\mathcal{C}_2)$.

\textit{Interdependence}: Value and risk share common determinants ($T$), establishing bidirectional influence.

\textit{Persistence}: Value maximization requires deploying technologies up to where marginal value equals marginal cost. Risk minimization requires either no deployment (contradicting value maximization) or perfect mitigation (infeasible in practice). No configuration simultaneously achieves both optima. \qed
\end{proof}

\subsection{Dynamics of Paradoxical Tension}
\label{subsec:propositions}

\begin{definition}[Tension Intensity]
\label{def:intensity}
\textbf{Tension intensity} $\Phi(t)$ at time $t$ is:
\begin{equation}
\Phi(t) = \left| \nabla V(\mathcal{C}_t) \cdot \nabla R(\mathcal{C}_t) \right| \cdot \mathbb{1}\left[\nabla V \cdot \nabla R > 0\right]
\label{eq:tension}
\end{equation}
Tension is high when value and risk increase together, indicating that improvement in one dimension necessarily affects the other similarly.
\end{definition}

\begin{proposition}[Tension Amplification Under Trade-off Logic]
\label{prop:amplification}
Organizations applying trade-off logic to paradoxical tensions experience monotonically increasing tension intensity over time.
\end{proposition}

\begin{proof}
Trade-off logic prescribes optimization:
\begin{equation}
\mathcal{C}_{t+1} = \arg\max_{\mathcal{C}} \left[ \lambda V(\mathcal{C}) - (1-\lambda) R(\mathcal{C}) \right] 
\end{equation}
However, in paradoxical contexts, the ``frontier'' is unstable due to environmental changes $\theta_t$. Organizations respond with lag $\tau > 0$, implementing configurations optimal for past conditions. The mismatch generates pressure for reconfiguration, but each attempt faces the same lag. Tension intensity evolves as:
\begin{equation}
\frac{d\Phi}{dt} = \left| \frac{\partial^2 V}{\partial \mathcal{C} \partial \theta} \cdot \frac{d\theta}{dt} \right| \cdot \left| \frac{\partial^2 R}{\partial \mathcal{C} \partial \theta} \cdot \frac{d\theta}{dt} \right| \cdot \left(1 - e^{-\tau/\tau_0}\right) > 0
\end{equation}
establishing monotonic increase. \qed
\end{proof}

\begin{remark}
Proposition~\ref{prop:amplification} explains the principles-to-practices gap: organizations applying trade-off logic perpetually chase a moving target, leading to frustration and inconsistency.
\end{remark}

\begin{proposition}[Value of Paradox Acceptance]
\label{prop:acceptance}
Organizations adopting paradox acceptance strategies achieve strictly higher long-run expected utility than those applying trade-off logic, provided environmental volatility exceeds a threshold $\sigma^*$.
\end{proposition}

The proof follows from showing that under trade-off logic, tension explosion drives utility to negative infinity, while paradox acceptance maintains bounded tension and thus finite utility.

\subsection{Paradox Management Strategies}
\label{subsec:strategies}

Building on paradox theory \cite{poole1989alternative, lewis2000exploring}, we formalize four management strategies.

\begin{definition}[Strategy Space]
\label{def:strategies}
The \textbf{paradox management strategy space} $\mathcal{S} = \{\mathcal{S}_A, \mathcal{S}_T, \mathcal{S}_S, \mathcal{S}_I\}$ comprises:
\begin{itemize}
    \item \textit{Acceptance} ($\mathcal{S}_A$): Acknowledging tension as inherent and productive
    \item \textit{Temporal Separation} ($\mathcal{S}_T$): Alternating emphasis between poles over time
    \item \textit{Spatial Separation} ($\mathcal{S}_S$): Differentiating configurations across contexts
    \item \textit{Integration} ($\mathcal{S}_I$): Synthesizing opposing elements into novel configurations
\end{itemize}
\end{definition}

\begin{theorem}[Optimal Strategy Selection]
\label{thm:optimal}
The optimal paradox management strategy depends on organizational and environmental conditions:
\begin{enumerate}
    \item $\mathcal{S}_A$ is optimal when environmental volatility $\sigma$ is high and organizational adaptation capacity $\kappa$ is low
    \item $\mathcal{S}_T$ is optimal when stakeholders have heterogeneous time horizons and the organization can credibly commit to cycles
    \item $\mathcal{S}_S$ is optimal when AI applications are modular with limited interdependencies and deployment contexts are sufficiently distinct
    \item $\mathcal{S}_I$ is optimal when the organization possesses high dynamic capabilities and can invest in innovation
\end{enumerate}
\end{theorem}

\begin{corollary}[Strategy Portfolio]
\label{cor:portfolio}
For organizations with diverse AI portfolios, optimal governance involves deploying different strategies across different contexts, matching strategy characteristics to contextual requirements.
\end{corollary}

\section{Research Methodology}
\label{sec:methodology}

This research employs a three-phase mixed methodology combining systematic literature review (SLR) with design science research (DSR). Figure~\ref{fig:research_design} presents an overview of our research design.

\begin{figure}[htbp]
\centering
\includegraphics[width=0.99\textwidth]{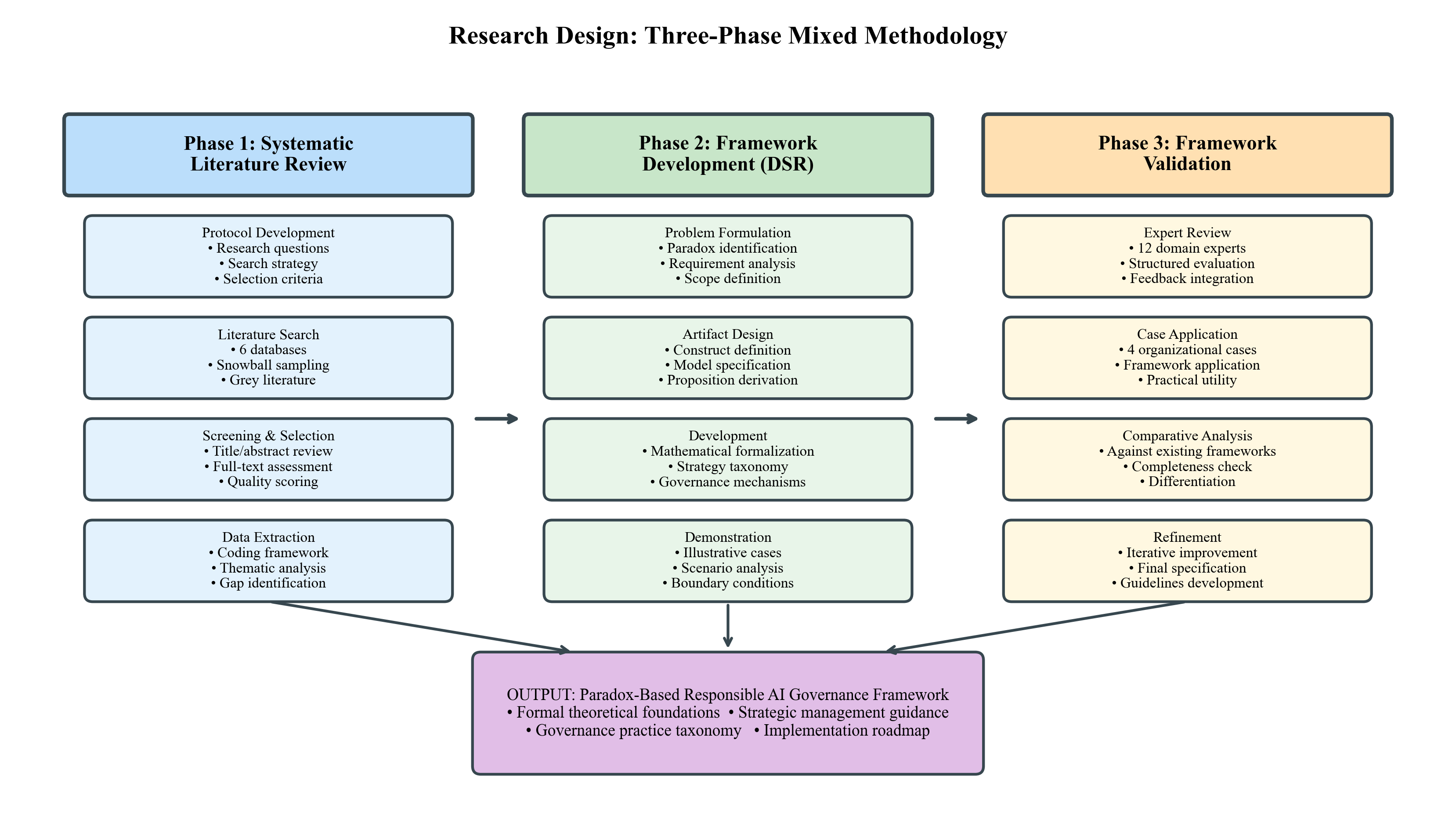}
\caption{Research Design: Three-Phase Mixed Methodology}
\label{fig:research_design}
\end{figure}

\subsection{Systematic Literature Review}

The SLR follows established guidelines \cite{kitchenham2007guidelines} and PRISMA 2020 standards \cite{page2021prisma}. We searched six databases (Web of Science, Scopus, ACM Digital Library, IEEE Xplore, AIS eLibrary, Google Scholar) using terms related to AI governance, responsible AI, and AI ethics. Inclusion criteria specified articles published 2018-2025 addressing AI governance with empirical evidence or theoretical contribution.

From 3,003 initial records, we identified 2,341 after deduplication. Title/abstract screening ($\kappa = 0.81$) and full-text assessment yielded 88 studies meeting quality thresholds ($\bar{Q} \geq 3.0$ on a 5-point scale across six criteria). Data extraction employed structured coding, with thematic analysis generating 247 first-order codes consolidated into 42 second-order themes and 6 theoretical dimensions. Theoretical saturation was achieved at approximately $n = 72$ articles.

\subsection{Framework Development}

Design science methodology guided framework development \cite{hevner2004design}. Following Peffers et al.'s \cite{peffers2007design} process model, we: (1) identified the problem (principles-to-practices gap), (2) defined objectives (a framework addressing paradoxical tensions), (3) designed and developed the PRAIG framework, (4) demonstrated application through case studies, and (5) evaluated through expert review.

Expert evaluation involved 12 specialists (4 academics, 4 practitioners, 4 policymakers) who assessed the framework on content validity, construct validity, internal consistency, practical utility, and comparative advantage, using 4-item subscales (7-point Likert). Results indicated strong support: content validity ($\bar{E} = 5.8$), construct validity ($\bar{E} = 5.6$), internal consistency ($\bar{E} = 6.1$), practical utility ($\bar{E} = 5.4$), and comparative advantage ($\bar{E} = 5.7$).

\section{Findings: The PRAIG Framework}
\label{sec:findings}

The Paradox-based Responsible AI Governance (PRAIG) framework comprises four components: taxonomies of benefits and risks, an integrated governance framework, and paradox management strategies.

\subsection{Taxonomy of Strategic AI Benefits}

Our synthesis identified three aggregate dimensions of AI benefits, as illustrated in Figure~\ref{fig:benefits}.

\begin{figure}[ht]
\centering
\includegraphics[width=0.95\textwidth]{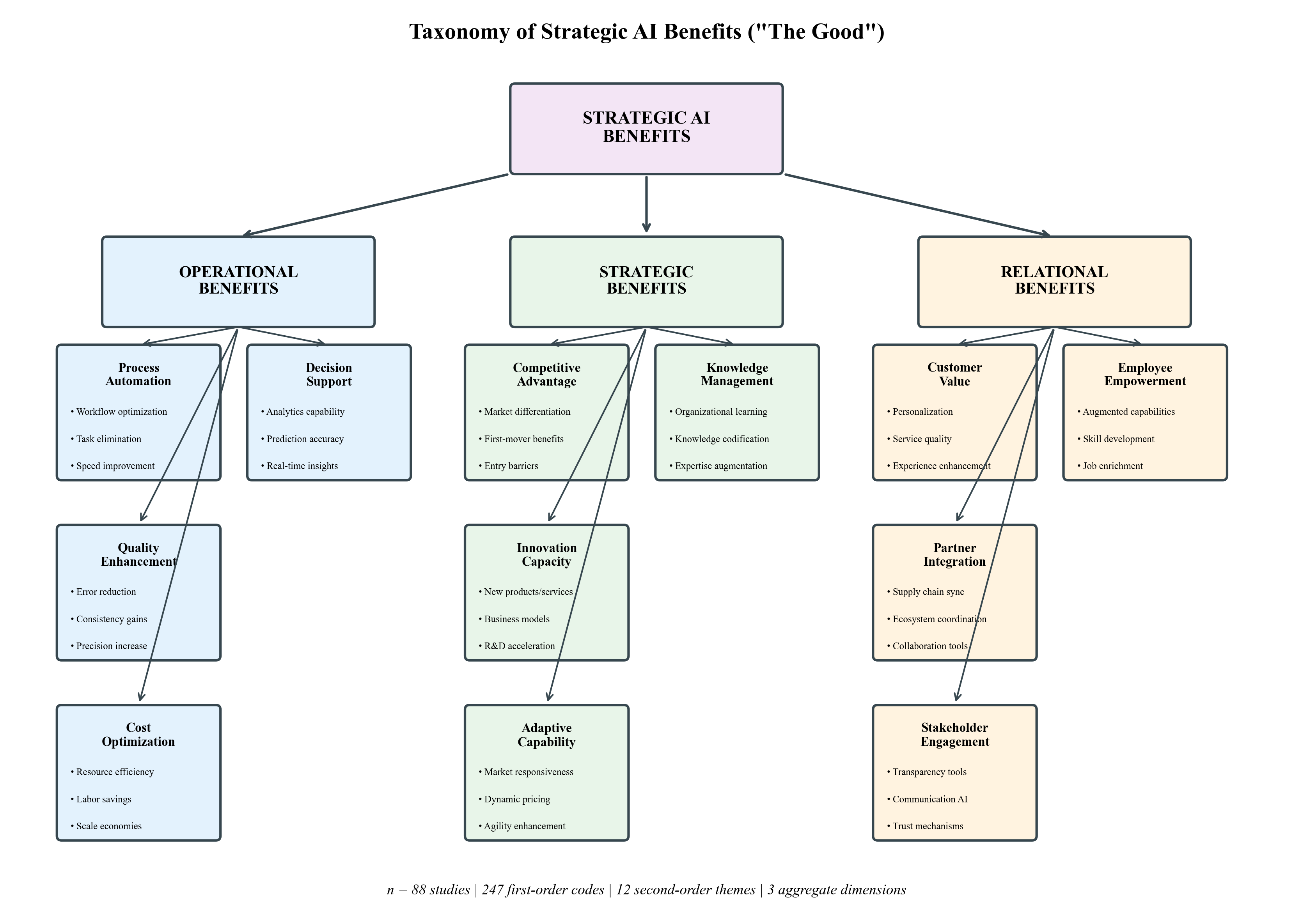}
\caption{Taxonomy of Strategic AI Benefits (``The Good'')}
\label{fig:benefits}
\end{figure}

\textit{Operational benefits} (38\% of codes) include process automation (15-40\% productivity gains), quality enhancement (25-50\% error reduction), cost optimization (10-30\% cost reduction), and decision support (20-35\% decision improvement).

\textit{Strategic benefits} (35\% of codes) encompass competitive advantage through differentiation and entry barriers, innovation capacity through accelerated R\&D and business model innovation, adaptive capability enabling agility in VUCA environments, and knowledge management transforming organizational learning.

\textit{Relational benefits} (27\% of codes) comprise customer value through personalization, partner integration through supply chain synchronization, stakeholder engagement through transparency tools, and employee empowerment through capability augmentation.

\subsection{Taxonomy of AI Risks}

Our analysis identified four aggregate dimensions of AI risks, as illustrated in Figure~\ref{fig:risks}.

\begin{figure}[ht]
\centering
\includegraphics[width=0.99\textwidth]{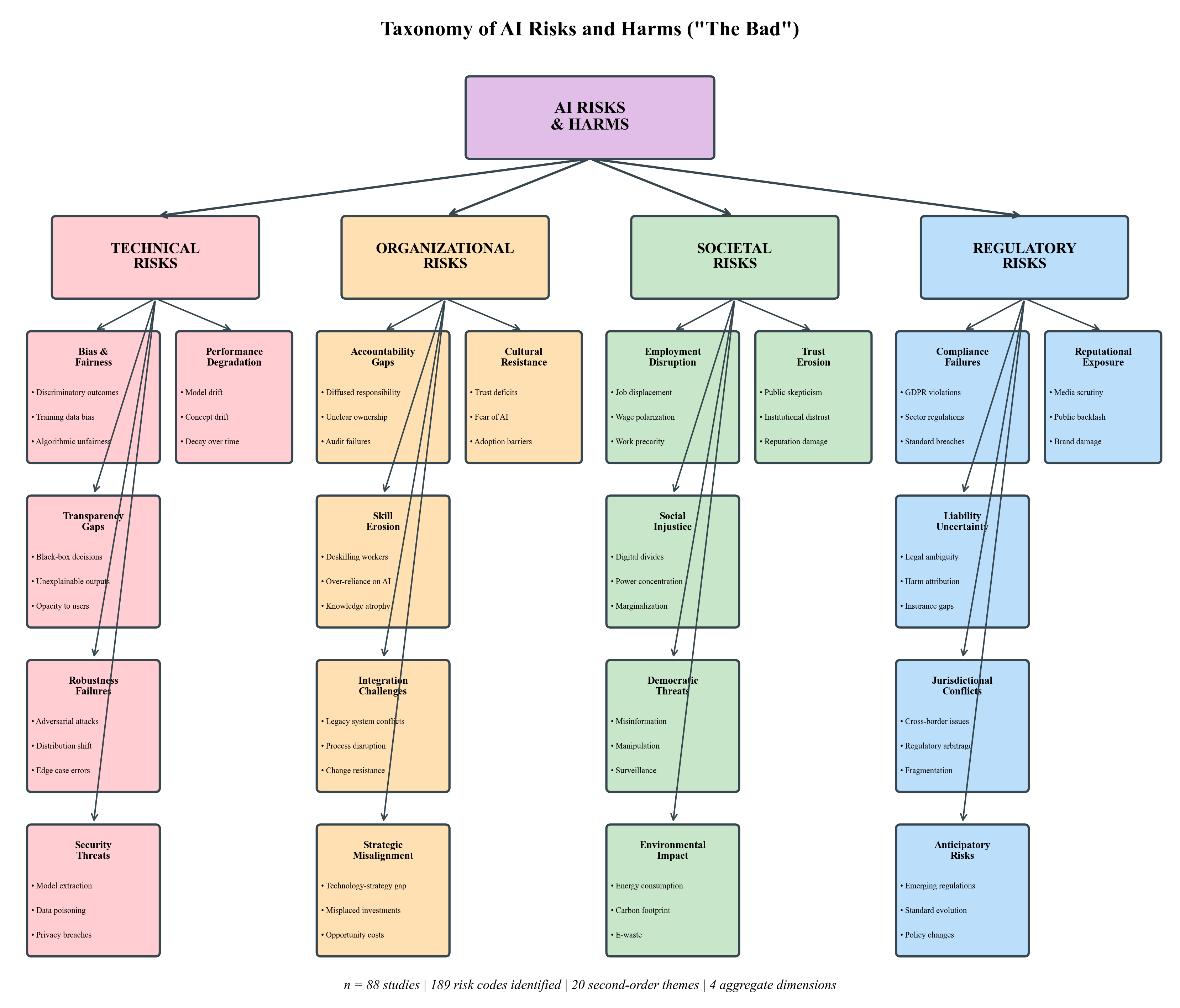}
\caption{Taxonomy of AI Risks and Harms (``The Bad'')}
\label{fig:risks}
\end{figure}

\textit{Technical risks} (32\% of codes) include bias and fairness violations, transparency gaps, robustness failures (adversarial attacks, distributional shift), security threats (model extraction, data poisoning), and performance degradation.

\textit{Organizational risks} (28\% of codes) encompass implementation failures, skill gaps, organizational resistance, strategic misalignment, and dependency risks.

\textit{Societal risks} (24\% of codes) comprise labor market disruption, inequality amplification, democratic processes erosion, environmental impacts, and power concentration.

\textit{Regulatory risks} (16\% of codes) include compliance failures, jurisdictional complexity, and liability uncertainty.

\subsection{Integrated Governance Framework}

The PRAIG framework links antecedents, practices, outcomes, and feedback dynamics, as illustrated in Figure~\ref{fig:framework}.

\begin{figure}[ht]
\centering
\includegraphics[width=0.99\textwidth]{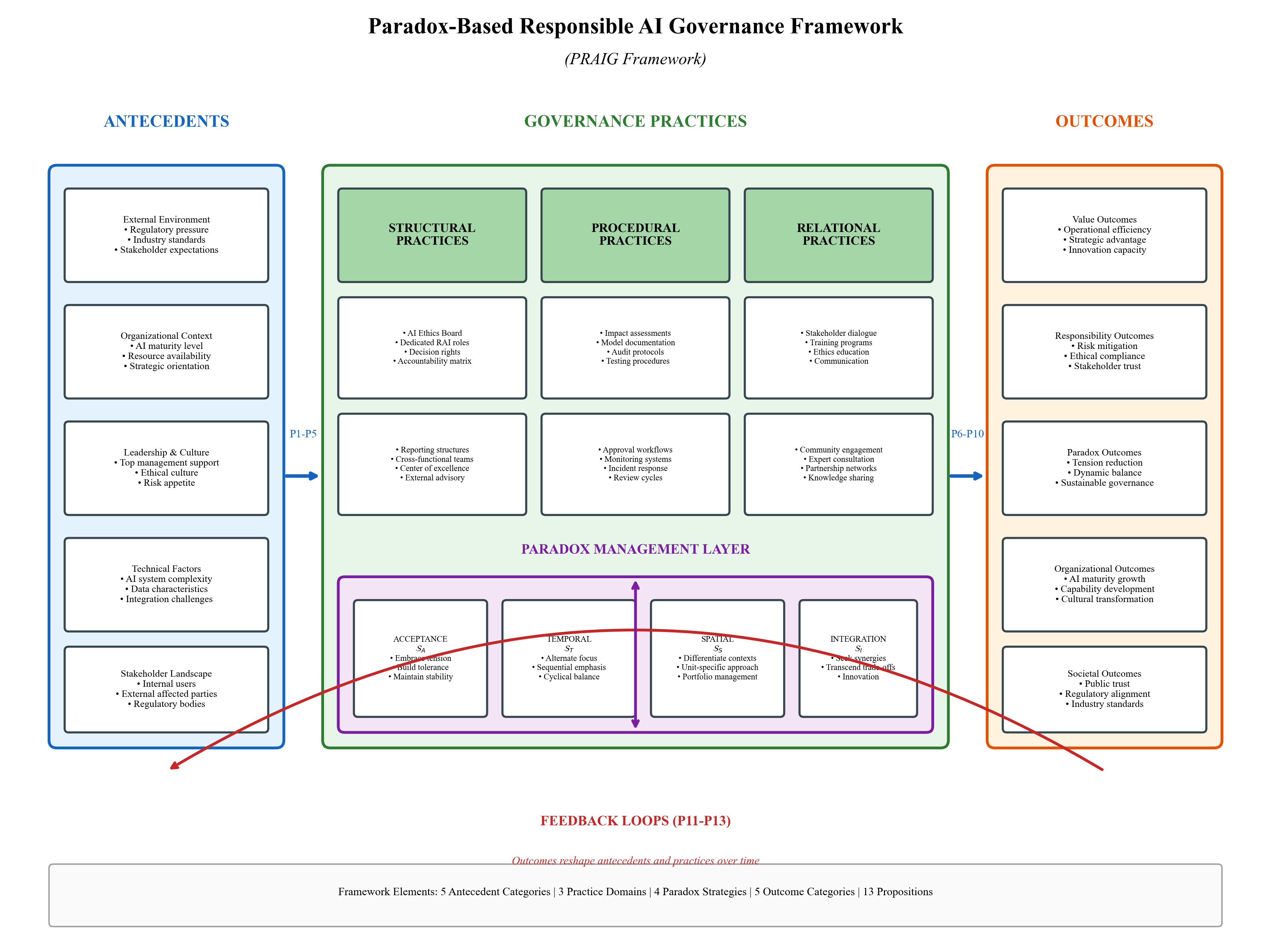}
\caption{The PRAIG Framework: Integrated Model of Responsible AI Governance}
\label{fig:framework}
\end{figure}

\subsubsection{Antecedents}

\textit{Organizational antecedents} include AI maturity level, organizational culture (risk tolerance, ethical orientation), dynamic capabilities, and resource availability. \textit{Environmental antecedents} encompass regulatory pressure, stakeholder expectations, competitive dynamics, and technology evolution rate.

\subsubsection{Governance Practices}

\textit{Structural practices} establish formal organizational elements: AI ethics boards, Chief AI Ethics Officer roles, dedicated governance units, and clear accountability structures.

\textit{Procedural practices} define governance processes: AI impact assessments, algorithmic auditing, documentation standards (model cards, datasheets), lifecycle governance, and incident response.

\textit{Relational practices} foster engagement and collaboration: stakeholder dialogue, training programs, ethics education, community engagement, and expert consultation.

\begin{proposition}[Governance Effectiveness]
\label{prop:governance_effectiveness}
Governance effectiveness is a function of structural ($G_S$), procedural ($G_P$), and relational ($G_R$) practice intensity with complementarity effects:
\begin{equation}
\text{Effectiveness} = G_S^{\alpha} \cdot G_P^{\beta} \cdot G_R^{\gamma} \cdot (1 + \delta \cdot G_S \cdot G_P \cdot G_R)
\label{eq:governance_effectiveness}
\end{equation}
The multiplicative structure implies that weakness in any domain limits overall effectiveness; the complementarity term reflects that practices reinforce each other.
\end{proposition}

\subsubsection{Outcomes and Feedback}

Effective governance produces value outcomes (deployment on solid foundations), responsibility outcomes (risk mitigation, ethical compliance), paradox outcomes (tension reduction, dynamic balance), organizational outcomes (capability development), and societal outcomes (public trust, regulatory alignment).

Three feedback loops drive governance evolution: a \textit{reinforcing loop} (R1) where positive outcomes build commitment and capability; a \textit{balancing loop} (B1) where negative outcomes trigger governance intensification; and a \textit{learning loop} (L1) where experience generates capability development.

\subsection{Paradox Management Implementation}

Figure~\ref{fig:strategies} illustrates the four paradox management strategies with their characteristics and optimal conditions.

\begin{figure}[ht]
\centering
\includegraphics[width=0.99\textwidth]{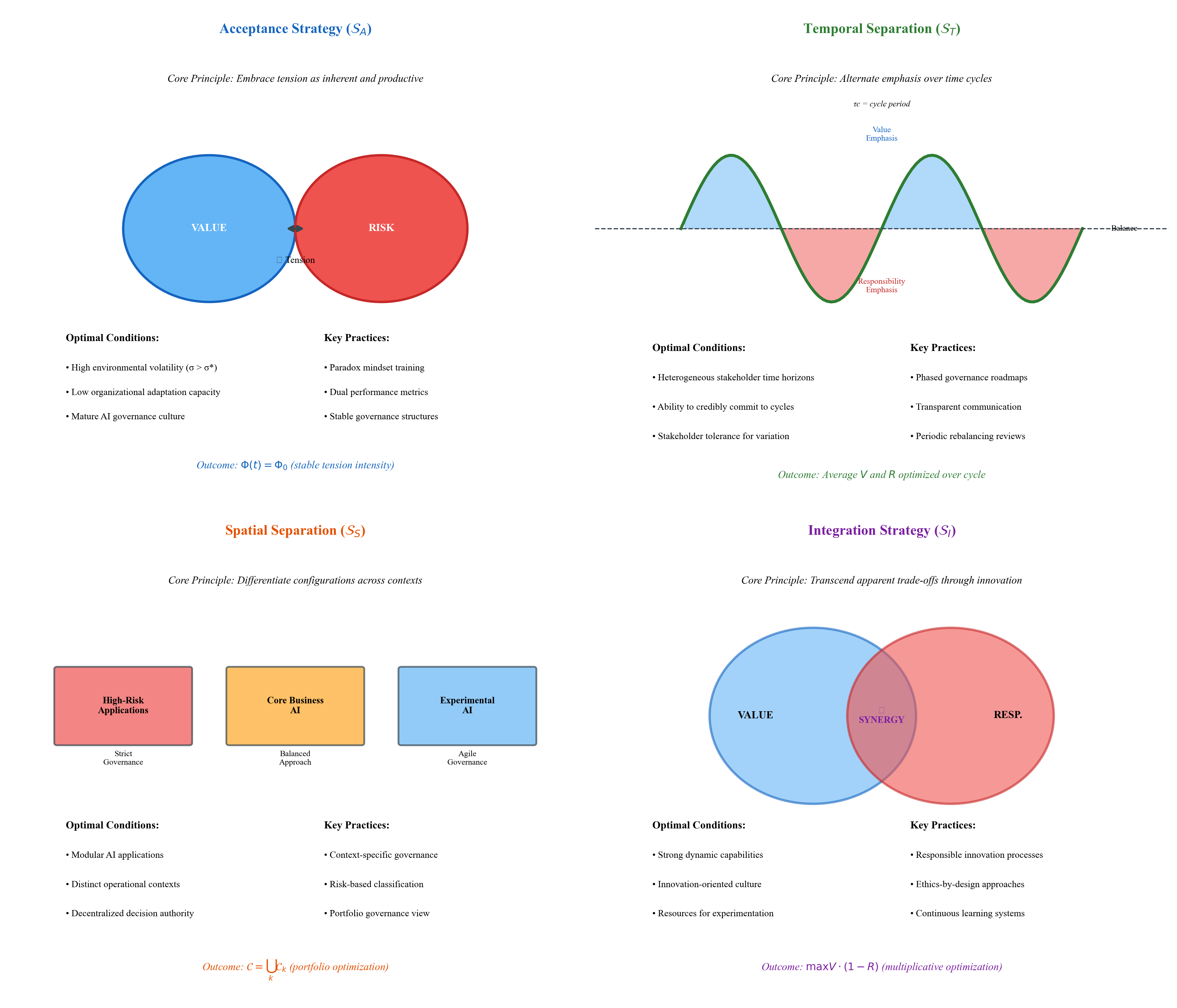}
\caption{Paradox Management Strategies: Characteristics and Conditions}
\label{fig:strategies}
\end{figure}

\subsubsection{Acceptance Strategy}

Acceptance involves acknowledging tension as inherent and productive. Key practices include paradox mindset development, dual performance metrics, leadership modeling, and dialogue spaces. Optimal conditions include high environmental volatility, limited adaptation capacity, and sufficient organizational maturity to tolerate ambiguity.

\subsubsection{Temporal Separation}

Temporal separation alternates emphasis over time cycles. Under cycle period $\tau_c$, governance emphasis oscillates: $\epsilon(t) = \frac{1}{2}[1 + \cos(2\pi t/\tau_c)]$. Key practices include phased roadmaps, transparent communication, periodic rebalancing, and stakeholder expectation management. Optimal when stakeholders have heterogeneous time horizons.

\subsubsection{Spatial Separation}

Spatial separation differentiates configurations across contexts. The organization maintains a portfolio $\{\mathcal{C}_k\}_{k=1}^{K}$ enabling high-value, higher-risk configurations in some contexts balanced by conservative configurations elsewhere. Table~\ref{tab:spatial_separation} illustrates application across risk categories aligned with the EU AI Act.

\begin{table}[htbp]
\centering
\caption{Spatial Separation: Risk-Based Governance Differentiation}
\label{tab:spatial_separation}
\small
\begin{tabular}{p{2.2cm}p{3cm}p{3cm}p{4cm}}
\hline
\textbf{Risk Category} & \textbf{Example Applications} & \textbf{Governance Intensity} & \textbf{Key Mechanisms} \\
\hline
Unacceptable Risk & Social scoring; Real-time biometric surveillance & Prohibited & Pre-deployment prohibition; Strict enforcement \\
\hline
High Risk & Credit decisions; Employment screening; Medical diagnosis & Maximum & Mandatory conformity assessment; Human oversight; Continuous monitoring \\
\hline
Limited Risk & Chatbots; Emotion recognition; Content recommendation & Moderate & Transparency requirements; User notification; Opt-out mechanisms \\
\hline
Minimal Risk & Spam filters; Inventory optimization; Game AI & Baseline & Voluntary codes; Internal review; Best practices \\
\hline
\end{tabular}
\end{table}

\subsubsection{Integration Strategy}

Integration synthesizes opposing elements through innovation. The organization solves:
\begin{equation}
\mathcal{C}^* = \arg\max_{\mathcal{C}} V(\mathcal{C}) \quad \text{s.t.} \quad R(\mathcal{C}) \leq \bar{R}, \quad \mathcal{C} \in \mathcal{F}(\mathcal{C}_V \cup \mathcal{C}_R)
\label{eq:integration}
\end{equation}
where $\mathcal{F}$ allows novel configurations combining elements from value-maximizing and risk-minimizing configurations. Key practices include responsible-AI-by-design, governance innovation, and synergy exploitation. Optimal when the organization possesses high dynamic capabilities.

\section{Discussion}
\label{sec:discussion}

\subsection{Theoretical Contributions}

Our research makes three primary contributions, summarized in Table~\ref{tab:theoretical_contributions}.

\begin{table}[htbp]
\centering
\caption{Summary of Theoretical Contributions}
\label{tab:theoretical_contributions}
\small
\begin{tabular}{p{3.2cm}p{4cm}p{2.8cm}p{3cm}}
\hline
\textbf{Contribution} & \textbf{Core Insight} & \textbf{Extends} & \textbf{Challenges} \\
\hline
Paradox reconceptualization & Value-responsibility relationship is paradoxical, not trade-off & Paradox theory to AI domain & Trade-off assumptions in AI governance \\
\hline
Strategy taxonomy & Four formalized strategies with contingency conditions & Domain-specific paradox strategies & One-size-fits-all governance approaches \\
\hline
PRAIG framework & Integrated model with feedback dynamics & Static governance frameworks & Principle-focused approaches lacking implementation theory \\
\hline
\end{tabular}
\end{table}

\textit{First}, we reconceptualize responsible AI governance as paradox management rather than trade-off optimization. The value-responsibility relationship exhibits contradiction with interdependence, persistence over time, and dynamic amplification under trade-off logic. Proposition~\ref{prop:amplification} explains the principles-to-practices gap as a predictable consequence of applying inappropriate logic to paradoxical situations.

\textit{Second}, we develop a taxonomy of paradox management strategies adapted to AI governance. We provide formal specification (mathematical characterization), contextual adaptation (AI-specific practices), and contingency conditions (factors determining effectiveness). The portfolio perspective (Corollary~\ref{cor:portfolio}) recognizes that effective governance requires deploying different strategies across contexts.

\textit{Third}, the PRAIG framework provides an integrated account linking antecedents, practices, and outcomes. Unlike static frameworks specifying requirements, PRAIG explicitly models feedback dynamics that explain governance evolution. The complementarity specification (Proposition~\ref{prop:governance_effectiveness}) advances understanding of how structural, procedural, and relational practices interact.

\subsection{Practical Implications}

For executives, our findings suggest: (1) \textit{Embrace paradox} rather than seeking resolution---develop organizational capacity for ongoing paradox management; (2) \textit{Match strategies to context}---assess conditions, classify applications, allocate appropriate strategies, and continuously adjust; (3) \textit{Invest in governance capability} as a strategic asset enabling aggressive yet responsible AI deployment; (4) \textit{Leverage feedback loops}---amplify reinforcing loops, learn from balancing loops, and strengthen learning loops.

For practitioners, implications include: designing for paradox rather than against it, embedding governance considerations from inception, and supporting spatial separation through modular design.

For policymakers, our analysis suggests regulatory approaches should emphasize governance capabilities and processes rather than static compliance targets, recognizing the dynamic and paradoxical nature of organizational challenges.

\subsection{Limitations}

We acknowledge limitations. The framework is primarily conceptual; while expert evaluation and case applications provide preliminary validation, large-scale empirical testing remains needed. Strategy contingencies are theoretically derived rather than empirically calibrated. The mathematical formalization, while enhancing precision, may limit accessibility for some practitioners. Boundary conditions include organizational size (sufficient scale needed), AI centrality (framework most applicable when AI is strategically important), and regulatory environment (some external pressure assumed).

\section{Research Agenda}
\label{sec:agenda}

The PRAIG framework motivates a structured research agenda across four themes, summarized in Table~\ref{tab:research_agenda}.

\begin{table}[htbp]
\centering
\caption{Research Agenda Overview}
\label{tab:research_agenda}
\small
\begin{tabular}{p{2.8cm}p{5.2cm}p{3.2cm}c}
\hline
\textbf{Theme} & \textbf{Research Questions} & \textbf{Methods} & \textbf{Priority} \\
\hline
Paradox Dynamics & RQ1: Tension measurement; RQ2: Amplification mechanisms; RQ3: Temporal evolution & Longitudinal surveys; Process studies; Simulation & High \\
\hline
Strategy Effectiveness & RQ4: Contingency validation; RQ5: Strategy combinations; RQ6: Transition dynamics & Field experiments; Comparative cases; Panel studies & High \\
\hline
Governance Mechanisms & RQ7: Practice complementarity; RQ8: Feedback loop dynamics & Multi-level analysis; System dynamics; Action research & Medium \\
\hline
Contextual Variation & RQ9: Cross-cultural differences; RQ10: Technology-specific governance & Cross-national surveys; Comparative cases; Design science & Medium \\
\hline
\end{tabular}
\end{table}

\textit{Theme 1: Paradox Dynamics.} RQ1 addresses how paradoxical tension can be measured and tracked. RQ2 investigates mechanisms driving tension amplification. RQ3 examines how tensions evolve over the governance lifecycle.

\textit{Theme 2: Strategy Effectiveness.} RQ4 seeks empirical validation of contingency conditions in Theorem~\ref{thm:optimal}. RQ5 investigates how strategies combine in portfolios. RQ6 examines organizational transitions between strategies.

\textit{Theme 3: Governance Mechanisms.} RQ7 tests the complementarity specification in Proposition~\ref{prop:governance_effectiveness}. RQ8 investigates how feedback loops shape governance evolution.

\textit{Theme 4: Contextual Variation.} RQ9 examines cross-cultural differences in responsible AI governance. RQ10 addresses governance challenges arising from specific AI technologies (generative AI, autonomous systems, foundation models).

\section{Conclusion}
\label{sec:conclusion}

This paper has addressed the fundamental challenge of how organizations can realize AI's strategic benefits while fulfilling responsibilities to stakeholders and society. Our analysis reveals that this challenge is not merely practical but theoretically significant, requiring reconceptualization.

The central insight is that the relationship between AI value creation and responsible deployment is paradoxical rather than merely conflictual. Unlike trade-offs amenable to optimization, paradoxes involve contradictory yet interdependent elements that persist over time. Our formal analysis demonstrates that applying trade-off logic to this paradoxical situation amplifies rather than resolves tensions, explaining the persistent principles-to-practices gap.

This reconceptualization yields the PRAIG framework, which provides: (1) a theoretically grounded explanation for governance challenges through the dynamics of tension amplification; (2) a taxonomy of paradox management strategies with formally specified contingencies; and (3) an integrated model linking antecedents, practices, outcomes, and feedback dynamics.

For executives, our findings suggest a fundamental shift: from seeking optimal balance to building organizational capacity for ongoing paradox management. This implies investing in paradox mindset development, designing governance systems that institutionalize dual objectives, and leveraging feedback loops for continuous adaptation.

As AI becomes increasingly embedded in organizational processes and societal systems, organizations that develop sophisticated capabilities for managing the value-responsibility paradox will capture AI's benefits while avoiding harms. The path forward requires abandoning the premise that responsible AI governance is primarily about finding the right balance. Instead, it demands embracing the inherent paradox---recognizing that value and responsibility are interdependent, that tensions must be managed rather than resolved, and that effective governance emerges from ongoing engagement with contradiction.

The good, the bad, and the AI are not separate phenomena but intertwined dimensions of a single complex challenge. Meeting this challenge is among the defining tasks of our technological age.

\section*{Acknowledgement}
Authors would like to thank 3S Holding O\"U for supporting this work financially. Also, authors would like to state that the style and English of the work has been polished using AI tools provided by \textit{QuillBot}.

\bibliographystyle{IEEEtran}
\bibliography{ref}

@article{mckinsey2023genai,
  author    = {{McKinsey Global Institute}},
  title     = {The Economic Potential of Generative {AI}: The Next Productivity Frontier},
  journal   = {McKinsey \& Company Research Report},
  year      = {2023},
  month     = {June},
  note      = {Available at: https://www.mckinsey.com/capabilities/mckinsey-digital/our-insights/the-economic-potential-of-generative-ai-the-next-productivity-frontier}
}

@article{ransbotham2020winning,
  author    = {Ransbotham, Sam and Khodabandeh, Shervin and Kiron, David and Candelon, Fran{\c{c}}ois and Chu, Michael and LaFountain, Burt},
  title     = {Winning With {AI}},
  journal   = {MIT Sloan Management Review},
  year      = {2020},
  volume    = {62},
  number    = {1},
  pages     = {1--17}
}

@article{enholm2022ai,
  author    = {Enholm, Ida Maria and Papagiannidis, Emmanouil and Mikalef, Patrick and Krogstie, John},
  title     = {Artificial Intelligence and Business Value: A Literature Review},
  journal   = {Information Systems Frontiers},
  year      = {2022},
  volume    = {24},
  number    = {5},
  pages     = {1709--1734},
  doi       = {10.1007/s10796-021-10186-w}
}

@incollection{dastin2022amazon,
  author    = {Dastin, Jeffrey},
  title     = {Amazon Scraps Secret {AI} Recruiting Tool That Showed Bias Against Women},
  booktitle = {Ethics of Data and Analytics},
  publisher = {Auerbach Publications},
  year      = {2022},
  pages     = {296--299},
  doi       = {10.1201/9781003278290}
}

@article{angwin2016machine,
  author    = {Angwin, Julia and Larson, Jeff and Mattu, Surya and Kirchner, Lauren},
  title     = {Machine Bias: There's Software Used Across the Country to Predict Future Criminals. And It's Biased Against Blacks},
  journal   = {ProPublica},
  year      = {2016},
  month     = {May},
  note      = {Available at: https://www.propublica.org/article/machine-bias-risk-assessments-in-criminal-sentencing}
}

@inproceedings{buolamwini2018gender,
  author    = {Buolamwini, Joy and Gebru, Timnit},
  title     = {Gender Shades: Intersectional Accuracy Disparities in Commercial Gender Classification},
  booktitle = {Proceedings of the 1st Conference on Fairness, Accountability and Transparency},
  year      = {2018},
  pages     = {77--91},
  publisher = {PMLR}
}

@article{mikalef2022thinking,
  author    = {Mikalef, Patrick and Conboy, Kieran and Lundstr{\"o}m, Jenny Eriksson and Popovi{\v{c}}, Ale{\v{s}}},
  title     = {Thinking Responsibly About Responsible {AI} and `The Dark Side' of {AI}},
  journal   = {European Journal of Information Systems},
  year      = {2022},
  volume    = {31},
  number    = {3},
  pages     = {257--268},
  doi       = {10.1080/0960085X.2022.2026621}
}

@article{floridi2018ai4people,
  author    = {Floridi, Luciano and Cowls, Josh and Beltrametti, Monica and Chatila, Raja and Chazerand, Patrice and Dignum, Virginia and Luetge, Christoph and Madelin, Robert and Pagallo, Ugo and Rossi, Francesca and others},
  title     = {{AI4People}---An Ethical Framework for a Good {AI} Society: Opportunities, Risks, Principles, and Recommendations},
  journal   = {Minds and Machines},
  year      = {2018},
  volume    = {28},
  number    = {4},
  pages     = {689--707},
  doi       = {10.1007/s11023-018-9482-5}
}

@misc{euaiact2024,
  author    = {{European Parliament and Council of the European Union}},
  title     = {Regulation ({EU}) 2024/1689 Laying Down Harmonised Rules on Artificial Intelligence ({AI} Act)},
  year      = {2024},
  note      = {Official Journal of the European Union, L Series}
}

@article{papagiannidis2025responsible,
  author    = {Papagiannidis, Emmanouil and Mikalef, Patrick and Conboy, Kieran},
  title     = {Responsible Artificial Intelligence Governance: A Review and Research Framework},
  journal   = {Journal of Strategic Information Systems},
  year      = {2025},
  volume    = {34},
  pages     = {101885},
  doi       = {10.1016/j.jsis.2024.101885}
}

@article{mehrabi2021survey,
  author    = {Mehrabi, Ninareh and Morstatter, Fred and Saxena, Nripsuta and Lerman, Kristina and Galstyan, Aram},
  title     = {A Survey on Bias and Fairness in Machine Learning},
  journal   = {ACM Computing Surveys},
  year      = {2021},
  volume    = {54},
  number    = {6},
  pages     = {1--35},
  doi       = {10.1145/3457607}
}

@article{mittelstadt2019principles,
  author    = {Mittelstadt, Brent},
  title     = {Principles Alone Cannot Guarantee Ethical {AI}},
  journal   = {Nature Machine Intelligence},
  year      = {2019},
  volume    = {1},
  number    = {11},
  pages     = {501--507},
  doi       = {10.1038/s42256-019-0114-4}
}

@article{schiff2021principles,
  author    = {Schiff, Daniel and Biddle, Justin and Borenstein, Jason and Laas, Kelly},
  title     = {What's Next for {AI} Ethics, Policy, and Governance? A Global Overview},
  booktitle = {Proceedings of the 2020 AAAI/ACM Conference on AI, Ethics, and Society},
  year      = {2020},
  pages     = {153--158},
  doi       = {10.1145/3375627.3375804}
}

@article{mikalef2021artificial,
  author    = {Mikalef, Patrick and Gupta, Manjul},
  title     = {Artificial Intelligence Capability: Conceptualization, Measurement Calibration, and Empirical Study on Its Impact on Organizational Creativity and Firm Performance},
  journal   = {Information \& Management},
  year      = {2021},
  volume    = {58},
  number    = {3},
  pages     = {103434},
  doi       = {10.1016/j.im.2021.103434}
}

@article{bharadwaj2000resource,
  author    = {Bharadwaj, Anandhi S.},
  title     = {A Resource-Based Perspective on Information Technology Capability and Firm Performance: An Empirical Investigation},
  journal   = {MIS Quarterly},
  year      = {2000},
  volume    = {24},
  number    = {1},
  pages     = {169--196},
  doi       = {10.2307/3250983}
}

@article{davenport2018artificial,
  author    = {Davenport, Thomas H. and Ronanki, Rajeev},
  title     = {Artificial Intelligence for the Real World},
  journal   = {Harvard Business Review},
  year      = {2018},
  volume    = {96},
  number    = {1},
  pages     = {108--116}
}

@book{agrawal2018prediction,
  author    = {Agrawal, Ajay and Gans, Joshua and Goldfarb, Avi},
  title     = {Prediction Machines: The Simple Economics of Artificial Intelligence},
  publisher = {Harvard Business Review Press},
  year      = {2018},
  address   = {Boston, MA}
}

@article{huang2019artificial,
  author    = {Huang, Ming-Hui and Rust, Roland T.},
  title     = {Artificial Intelligence in Service},
  journal   = {Journal of Service Research},
  year      = {2018},
  volume    = {21},
  number    = {2},
  pages     = {155--172},
  doi       = {10.1177/1094670517752459}
}

@article{sjoedin2021ai,
  author    = {Sj{\"o}din, David and Parida, Vinit and Palmie, Maximilian and Wincent, Joakim},
  title     = {How {AI} Capabilities Enable Business Model Innovation: Scaling {AI} Through Co-Evolutionary Processes and Feedback Loops},
  journal   = {Journal of Business Research},
  year      = {2021},
  volume    = {134},
  pages     = {574--587},
  doi       = {10.1016/j.jbusres.2021.05.009}
}

@article{sambamurthy2003shaping,
  author    = {Sambamurthy, Vallabh and Bharadwaj, Anandhi and Grover, Varun},
  title     = {Shaping Agility Through Digital Options: Reconceptualizing the Role of Information Technology in Contemporary Firms},
  journal   = {MIS Quarterly},
  year      = {2003},
  volume    = {27},
  number    = {2},
  pages     = {237--263},
  doi       = {10.2307/30036530}
}

@article{verganti2020innovation,
  author    = {Verganti, Roberto and Vendraminelli, Luca and Iansiti, Marco},
  title     = {Innovation and Design in the Age of Artificial Intelligence},
  journal   = {Journal of Product Innovation Management},
  year      = {2020},
  volume    = {37},
  number    = {3},
  pages     = {212--227},
  doi       = {10.1111/jpim.12523}
}

@article{canhoto2021artificial,
  author    = {Canhoto, Ana Isabel and Clear, Fintan},
  title     = {Artificial Intelligence and Machine Learning as Business Tools: A Framework for Diagnosing Value Destruction Potential},
  journal   = {Business Horizons},
  year      = {2020},
  volume    = {63},
  number    = {2},
  pages     = {183--193},
  doi       = {10.1016/j.bushor.2019.11.003}
}

@article{obermeyer2019dissecting,
  author    = {Obermeyer, Ziad and Powers, Brian and Vogeli, Christine and Mullainathan, Sendhil},
  title     = {Dissecting Racial Bias in an Algorithm Used to Manage the Health of Populations},
  journal   = {Science},
  year      = {2019},
  volume    = {366},
  number    = {6464},
  pages     = {447--453},
  doi       = {10.1126/science.aax2342}
}

@article{burrell2016machine,
  author    = {Burrell, Jenna},
  title     = {How the Machine `Thinks': Understanding Opacity in Machine Learning Algorithms},
  journal   = {Big Data \& Society},
  year      = {2016},
  volume    = {3},
  number    = {1},
  pages     = {1--12},
  doi       = {10.1177/2053951715622512}
}

@article{arrieta2020explainable,
  author    = {Arrieta, Alejandro Barredo and D{\'i}az-Rodr{\'i}guez, Natalia and Del Ser, Javier and Bennetot, Adrien and Tabik, Siham and Barbado, Alberto and Garcia, Salvador and Gil-Lopez, Sergio and Molina, Daniel and Benjamins, Richard and others},
  title     = {Explainable Artificial Intelligence ({XAI}): Concepts, Taxonomies, Opportunities and Challenges Toward Responsible {AI}},
  journal   = {Information Fusion},
  year      = {2020},
  volume    = {58},
  pages     = {82--115},
  doi       = {10.1016/j.inffus.2019.12.012}
}

@article{rudin2019stop,
  author    = {Rudin, Cynthia},
  title     = {Stop Explaining Black Box Machine Learning Models for High Stakes Decisions and Use Interpretable Models Instead},
  journal   = {Nature Machine Intelligence},
  year      = {2019},
  volume    = {1},
  number    = {5},
  pages     = {206--215},
  doi       = {10.1038/s42256-019-0048-x}
}

@book{dignum2019responsible,
  author    = {Dignum, Virginia},
  title     = {Responsible Artificial Intelligence: How to Develop and Use {AI} in a Responsible Way},
  publisher = {Springer},
  year      = {2019},
  address   = {Cham, Switzerland},
  doi       = {10.1007/978-3-030-30371-6}
}

@article{nissenbaum1996accountability,
  author    = {Nissenbaum, Helen},
  title     = {Accountability in a Computerized Society},
  journal   = {Science and Engineering Ethics},
  year      = {1996},
  volume    = {2},
  number    = {1},
  pages     = {25--42},
  doi       = {10.1007/BF02639315}
}

@inproceedings{goodfellow2014explaining,
  author    = {Goodfellow, Ian J. and Shlens, Jonathon and Szegedy, Christian},
  title     = {Explaining and Harnessing Adversarial Examples},
  booktitle = {International Conference on Learning Representations (ICLR)},
  year      = {2015}
}

@inproceedings{quinonero2008dataset,
  author    = {Qui{\~n}onero-Candela, Joaquin and Sugiyama, Masashi and Schwaighofer, Anton and Lawrence, Neil D.},
  title     = {Dataset Shift in Machine Learning},
  booktitle = {Dataset Shift in Machine Learning},
  publisher = {MIT Press},
  year      = {2009},
  address   = {Cambridge, MA}
}

@article{gabriel2020artificial,
  author    = {Gabriel, Iason},
  title     = {Artificial Intelligence, Values, and Alignment},
  journal   = {Minds and Machines},
  year      = {2020},
  volume    = {30},
  number    = {3},
  pages     = {411--437},
  doi       = {10.1007/s11023-020-09539-2}
}

@book{solove2013privacy,
  author    = {Solove, Daniel J.},
  title     = {Nothing to Hide: The False Tradeoff Between Privacy and Security},
  publisher = {Yale University Press},
  year      = {2011},
  address   = {New Haven, CT}
}

@book{crawford2021atlas,
  author    = {Crawford, Kate},
  title     = {Atlas of {AI}: Power, Politics, and the Planetary Costs of Artificial Intelligence},
  publisher = {Yale University Press},
  year      = {2021},
  address   = {New Haven, CT}
}

@article{schwartz2020green,
  author    = {Schwartz, Roy and Dodge, Jesse and Smith, Noah A. and Etzioni, Oren},
  title     = {Green {AI}},
  journal   = {Communications of the ACM},
  year      = {2020},
  volume    = {63},
  number    = {12},
  pages     = {54--63},
  doi       = {10.1145/3381831}
}

@article{jobin2019global,
  author    = {Jobin, Anna and Ienca, Marcello and Vayena, Effy},
  title     = {The Global Landscape of {AI} Ethics Guidelines},
  journal   = {Nature Machine Intelligence},
  year      = {2019},
  volume    = {1},
  number    = {9},
  pages     = {389--399},
  doi       = {10.1038/s42256-019-0088-2}
}

@article{smith2011toward,
  author    = {Smith, Wendy K. and Lewis, Marianne W.},
  title     = {Toward a Theory of Paradox: A Dynamic Equilibrium Model of Organizing},
  journal   = {Academy of Management Review},
  year      = {2011},
  volume    = {36},
  number    = {2},
  pages     = {381--403},
  doi       = {10.5465/amr.2009.0223}
}

@article{schad2016paradox,
  author    = {Schad, Jonathan and Lewis, Marianne W. and Raisch, Sebastian and Smith, Wendy K.},
  title     = {Paradox Research in Management Science: Looking Back to Move Forward},
  journal   = {Academy of Management Annals},
  year      = {2016},
  volume    = {10},
  number    = {1},
  pages     = {5--64},
  doi       = {10.1080/19416520.2016.1162422}
}

@article{poole1989alternative,
  author    = {Poole, Marshall Scott and Van de Ven, Andrew H.},
  title     = {Using Paradox to Build Management and Organization Theories},
  journal   = {Academy of Management Review},
  year      = {1989},
  volume    = {14},
  number    = {4},
  pages     = {562--578},
  doi       = {10.5465/amr.1989.4308389}
}

@article{lewis2000exploring,
  author    = {Lewis, Marianne W.},
  title     = {Exploring Paradox: Toward a More Comprehensive Guide},
  journal   = {Academy of Management Review},
  year      = {2000},
  volume    = {25},
  number    = {4},
  pages     = {760--776},
  doi       = {10.5465/amr.2000.3707712}
}

@article{kitchenham2007guidelines,
  author    = {Kitchenham, Barbara and Charters, Stuart},
  title     = {Guidelines for Performing Systematic Literature Reviews in Software Engineering},
  journal   = {Technical Report EBSE-2007-01, Keele University},
  year      = {2007}
}

@article{page2021prisma,
  author    = {Page, Matthew J. and McKenzie, Joanne E. and Bossuyt, Patrick M. and Boutron, Isabelle and Hoffmann, Tammy C. and Mulrow, Cynthia D. and others},
  title     = {The {PRISMA} 2020 Statement: An Updated Guideline for Reporting Systematic Reviews},
  journal   = {BMJ},
  year      = {2021},
  volume    = {372},
  pages     = {n71},
  doi       = {10.1136/bmj.n71}
}

@article{hevner2004design,
  author    = {Hevner, Alan R. and March, Salvatore T. and Park, Jinsoo and Ram, Sudha},
  title     = {Design Science in Information Systems Research},
  journal   = {MIS Quarterly},
  year      = {2004},
  volume    = {28},
  number    = {1},
  pages     = {75--105},
  doi       = {10.2307/25148625}
}

@article{jafari2025mathematical,
  title={A Mathematical Framework for AI Singularity: Conditions, Bounds, and Control of Recursive Improvement},
  author={Jafari, Akbar Anbar and Ozcinar, Cagri and Anbarjafari, Gholamreza},
  journal={arXiv preprint arXiv:2511.10668},
  year={2025}
}

@article{domnich2021responsible,
  title={Responsible AI: Gender bias assessment in emotion recognition},
  author={Domnich, Artem and Anbarjafari, Gholamreza},
  journal={arXiv preprint arXiv:2103.11436},
  year={2021}
}

@article{sham2023ethical,
  title={Ethical AI in facial expression analysis: racial bias},
  author={Sham, Abdallah Hussein and Aktas, Kadir and Rizhinashvili, Davit and Kuklianov, Danila and Alisinanoglu, Fatih and Ofodile, Ikechukwu and Ozcinar, Cagri and Anbarjafari, Gholamreza},
  journal={Signal, Image and Video Processing},
  volume={17},
  number={2},
  pages={399--406},
  year={2023},
  publisher={Springer}
}

@article{mirza2025quantifying,
  title={Quantifying Gender Bias in Large Language Models Using Information-Theoretic and Statistical Analysis},
  author={Mirza, Imran and Jafari, Akbar Anbar and Ozcinar, Cagri and Anbarjafari, Gholamreza},
  journal={Information},
  volume={16},
  number={5},
  pages={358},
  year={2025},
  publisher={MDPI}
}

@article{peffers2007design,
  author    = {Peffers, Ken and Tuunanen, Tuure and Rothenberger, Marcus A. and Chatterjee, Samir},
  title     = {A Design Science Research Methodology for Information Systems Research},
  journal   = {Journal of Management Information Systems},
  year      = {2007},
  volume    = {24},
  number    = {3},
  pages     = {45--77},
  doi       = {10.2753/MIS0742-1222240302}
}

@article{topol2019high,
  author    = {Topol, Eric J.},
  title     = {High-Performance Medicine: The Convergence of Human and Artificial Intelligence},
  journal   = {Nature Medicine},
  year      = {2019},
  volume    = {25},
  number    = {1},
  pages     = {44--56},
  doi       = {10.1038/s41591-018-0300-7}
}

@article{cao2022ai,
  author    = {Cao, Longbing},
  title     = {{AI} in Finance: Challenges, Techniques, and Opportunities},
  journal   = {ACM Computing Surveys},
  year      = {2022},
  volume    = {55},
  number    = {3},
  pages     = {1--38},
  doi       = {10.1145/3502289}
}

\end{document}